\documentclass[letterpaper]{article} 
\usepackage{aaai25}  
\usepackage{times}  
\usepackage{helvet}  
\usepackage{courier}  
\usepackage[hyphens]{url}  
\usepackage{graphicx} 
\usepackage{natbib}  
\usepackage{caption} 
\usepackage{algorithm}
\usepackage[noend]{algorithmic}

\usepackage{newfloat}
\usepackage{listings}
\DeclareCaptionStyle{ruled}{labelfont=normalfont,labelsep=colon,strut=off} 
\floatstyle{ruled}
\newfloat{listing}{tb}{lst}{}
\floatname{listing}{Listing}
\usepackage{papercmds}
\SetSymbolFont{stmry}{bold}{U}{stmry}{m}{n}

\usepackage{mathtools} 
\usepackage{booktabs}

\usetikzlibrary{arrows.meta}

\tikzset{%
    my line width/.style={line width=.5pt},
    state common/.style={draw, my line width, rectangle, rounded corners,
        align=center},
    goal/.style={double, double distance=.8pt},
    trans/.style={-{>[length=3pt]}, my line width},
    emph trans/.style={trans, -{>[length=4pt]}, line width=1.5pt},
    trans label/.style={fill=white, circle, inner sep=0pt},
    to goal/.style={shorten >=.6pt},
}

\usepackage[inline]{enumitem}
\setlist[enumerate]{label=(\alph*)}

\title{A Formalism for Optimal Search with Dynamic Heuristics (Extended Version)}
\author{
    Remo Christen,
    Florian Pommerening,
    Clemens B\"uchner,
    Malte Helmert
}
\affiliations{
    University of Basel, Switzerland\\
    \{remo.christen, florian.pommerening, clemens.buechner, malte.helmert\}@unibas.ch
}

\begin{document}

\maketitle
\begin{abstract}
While most heuristics studied in heuristic search depend only on the state, some
accumulate information during search and thus also depend on the search history.
Multiple existing approaches use such dynamic heuristics in $\astar$-like
algorithms and appeal to classic results for $\astar$ to show that they return
optimal solutions. However, doing so disregards the intricacies of searching
with a mutable heuristic. We treat dynamic heuristics formally and propose a
framework that defines how the information dynamic heuristics rely on can be
modified. We use these transformations in a generic search algorithm and an
instantiation that models $\astar$ with dynamic heuristics, allowing us to
provide general conditions for optimality. We show that existing approaches fit
our framework and apply our results. Doing so for future applications of dynamic
heuristics may simplify formal arguments for optimality.

\end{abstract}
\section{Introduction}

Heuristic search with $\astar$ is a canonical approach to finding optimal
solutions in transition systems. Heuristic functions evaluate states to guide
the search and typically map states to numeric values. In this work, we consider
the more general class of \emph{dynamic heuristics} that additionally depend on
information procured during search
\citep[\eg{},][]{gelperin-aij1977,mero-aij1984,koyfman-et-al-aamas2024}.

Classic optimality results for $\astar$
\citep[\eg{},][]{hart-et-al-ieeessc1968, dechter-pearl-jacm1985} cannot be
directly applied to search with such dynamic heuristics, as their proofs assume
static heuristics.
Nevertheless, existing approaches to classical planning that use dynamic
heuristics, such as LM-$\astar$ \citep{karpas-domshlak-ijcai2009}, LTL-$\astar$
\citep{simon-roeger-socs2015}, or online abstraction refinement
\citep{eifler-fickert-socs2018, franco-torralba-socs2019}, are claimed to return
optimal solutions by referring to static notions of admissibility and thereby
implicitly relying on classic results, without taking the interaction between
search and heuristic into account. This is further complicated by discussions of
additional properties such as monotonically increasing heuristic values and
$\astar$ with re-evaluation, a modification that re-inserts states popped from
the open list if their heuristic value improved while they where queued.
Connections to consistency and reopening are implied, but their impact on the
optimality of solutions is not formally proven.

\citet{koyfman-et-al-aamas2024} show more formally that their $\astar$-based
approach guarantees optimal solutions while using a dynamic heuristic. They
achieve this result by showing that their heuristic has a property called
path-dynamic admissibility (PDA). However, PDA has strong requirements about the
behavior of the search. Therefore the complexity inherent to the interaction
between a dynamic heuristic and the search must be addressed while showing that
the heuristic has the PDA property, making it difficult to translate the result
to other approaches.

We aim to prove these results in a more general fashion by first formalizing the
information that dynamic heuristics rely on as a mutable object together with
functions to transform that object. We then define a generic algorithm framework
for heuristic forward search that models when information can be transformed,
define an instantiation of the framework based on $\astar$, and show under what
conditions it guarantees optimal solutions. These conditions rely on extensions
of classic heuristic properties to the dynamic case that do not depend on search
behavior. Additionally, we show conditions where no reopening occurs in the algorithm.

Finally, we apply these results to classical planning approaches that use dynamic
heuristics, such as online abstraction refinement where abstractions are improved
while the search is running. We also briefly consider other search strategies that
do not use dynamic heuristics as such, but whose behavior can be modeled by one.
This includes ways to consider multiple heuristics
\citep[\eg{},][]{zhang-bacchus-aaai2012, domshlak-et-al-jair2012,
tolpin-et-al-ijcai2013}, pathmax \citep{mero-aij1984}, path-dependent $f$-values
\citep{dechter-pearl-jacm1985}, deferred evaluation \citep{helmert-jair2006},
and partial expansion $\astar$ \citep{yoshizumi-et-al-aaai2000}.

\section{Transition Systems}

We consider transition systems $\Tsystem = \tuple{\Tstates, \Tlabels,
\Tcost, \Ttransitions, \Tinit, \Tgoal}$, where $\Tstates$ is a finite set of
\emph{states}; $\Tlabels$ is a finite set of \emph{labels}; $\Tcost\colon \Tlabels \to
\nonnegreals$ is a \emph{cost function} assigning each label a cost; the set
$\Ttransitions \subseteq \Tstates \times \Tlabels \times \Tstates$ contains
labeled \emph{transitions} $\tuple{s, \ell, s'}$; $\Tinit \in \Tstates$ is the
\emph{initial state}; and $\Tgoal \subseteq \Tstates$ are \emph{goal
states}. For a transition $\tuple{s, \ell, s'} \in \Ttransitions$, we call $s$
the \emph{origin}, $\ell$ the \emph{label}, and $s'$ the \emph{target} of the
transition, and call $s'$ a \emph{successor} of $s$.

A \emph{path} from $s$ to $s'$ is a sequence of transitions $\pi = \tuple{t_1,
\dots t_n}$ where the origin of $t_1$ is $s$, the origin of $t_{i+1}$ is the
target of $t_i$ for $1\le i<n$, and the target of $t_n$ is $s'$. A path to $s'$
is a path from $\Tinit$ to $s'$ and a \emph{solution} (for $s$) is a path (from~$s$) to some state in
$\Tgoal$. The cost of a path $\pi = \tuple{t_1, \dots t_n}$ with $t_i =
\tuple{s_i, \ell_i, s_i'}$ is $\Tcost(\pi) = \sum_{i=1}^n \Tcost(\ell_i)$. A
path (from $s$) to $s'$ is \emph{optimal}, if it has a minimal cost among all
paths (from $s$) to $s'$. We denote the cost of an optimal path to $s$ by
$\gstar(s)$. We say a state is \emph{reachable} if a path to it exists and that
$\Tsystem$ is \emph{solvable} if it has a solution.

We assume that the transition system is encoded in a compact form (\eg, a
planning task), where we can access the initial state, can generate the
successors of a given state, and can check whether a given state is a goal
state.

A (static) \emph{heuristic} is a function $h\colon \Tstates \to \nonnegreals
\cup \{\infty\}$. The \emph{perfect heuristic} $\hstar$ maps each state $s$ to the
cost of an optimal path from $s$ to a goal state or to $\infty$ if no such path
exists. A heuristic $h$ is \emph{admissible} if $h(s) \le \hstar(s)$ for all
$s\in \Tstates$, and \emph{consistent} if $h(s) \le \Tcost(\ell) + h(s')$ for all
$\tuple{s, \ell, s'} \in \Ttransitions$. A heuristic is \emph{safe} if $h(s) =
\infty$ implies $\hstar(s) = \infty$ and it is \emph{goal-aware} if $h(s) = 0$
for all goal states $s \in \Tgoal$.

\paragraph{Running Example}
In the following, we give intuitions for definitions and concepts based on the
transition system $\Tsystem = \tuple{\{A,B,C,D\}, \{x,y\}, \{x \mapsto 1, y
\mapsto 2\}, \Ttransitions, A, \{D\}}$ with $\Ttransitions = \{\tuple{A,x,B},
\tuple{A,y,C}, \tuple{B,y,C}, \tuple{C,x,D}\}$. Figure~\ref{fig:running-example}
visualizes $\Tsystem$. We also need the concept of \emph{landmarks} in our
examples \cite[e.g.,][]{hoffmann-et-al-jair2004}. In our context, a landmark for
a state $s$ is a label that occurs in every solution for $s$. For example, $x$
is a landmark for $C$, while $x$ and $y$ are landmarks for~$A$.

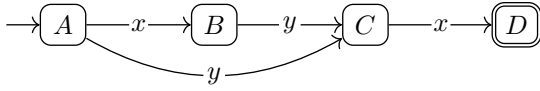
\begin{figure}
  \centering
  \begin{tikzpicture}
    \tikzset{%
        state/.style={state common, text width=3mm, inner sep=1.5mm},
    }
    \node[state] (a) at (0,0) {$A$};
    \node[state] (b) at (2,0) {$B$};
    \node[state] (c) at (4,0) {$C$};
    \node[state, goal] (d) at (6,0) {$D$};

    \draw[trans] (-7.5mm,0) to (a);
    \draw[trans] (a) to node[trans label] {$x$} (b);
    \draw[trans] (a) to[bend right] node[trans label] {$y$} (c);
    \draw[trans] (b) to node[trans label] {$y$} (c);
    \draw[trans, to goal] (c) to node[trans label] {$x$} (d);
\end{tikzpicture}
  \caption{Transition system of our running example.}
  \label{fig:running-example}
\end{figure}

\section{Dynamic Heuristics}

Dynamic heuristics depend on information gained while searching for a solution.
In our running example, this information is a partial function mapping states $s
\in \{A, B, C, D\}$ to sets of landmarks from $2^{\{x, y\}}$. We call the
set of all such functions the \emph{information space} of our landmark example.
In general, we consider dynamic heuristics to depend on information from an
arbitrary information space.

\citet{koyfman-et-al-aamas2024} introduced the notion of dynamic heuristics but
in their definition, the information can change arbitrarily between heuristic
evaluations. We extend their definition by considering more structured sources
of information where the possible modifications of information are limited.

\begin{definition}\label{def:info-src}
    An \emph{information source} $\src$ for a transition system with states
    $\Tstates$ and transitions $\Ttransitions$ consists of an information space
    $\infospace[\src]$, and the following constant and functions:
    \begin{align*}
        \Sinitinfo_{\src} &\in \infospace[\src]\text{,}\\
        \Supdate_{\src}&\colon \infospace[\src] \times \Ttransitions \to
            \infospace[\src]\text{ and} \\
        \Srefine_{\src}&\colon \infospace[\src] \times \Tstates \to
            \infospace[\src]\text{.}
    \end{align*}
\end{definition}

Information is initialized to $\Sinitinfo_{\src}$ and then modified with
$\Supdate_{\src}$ and $\Srefine_{\src}$. In our example, we could start in a
situation where we only know the landmarks of the initial state:
$\info_0 = \lminit = \{A \mapsto \{x,y\}\}$. This information can be updated by
considering transition $t = \tuple{A,y,C}$, thereby gaining information about
$C$: while $x$ must also be a landmark for $C$, the same is not true for $y$
because the label of $t$ is $y$. We can express this by $\info_u =
\lmupdate(\info_0, t) = \{A \mapsto \{x,y\}, C \mapsto \{x\}\}$. Alternatively,
we could refine $\info_0$, for example by computing landmarks for
$B$: $\info_r = \lmrefine(\info_0, B) = \{A \mapsto \{x,y\}, B \mapsto
\{x,y\}\}$. We consider expansion-based search algorithms, which can only refine
on known states and update along transitions starting in known states. Such an
algorithm can never refine $\info_0$ on $B$ without first updating on a
transition to $B$. The following definition formalizes this.

\begin{definition}\label{def:reachable-information}
    Let $\src$ be an information source for a transition system with states
    $\Tstates$ and transitions $\Ttransitions$. An information object $\info_n
    \in \infospace[\src]$ is called \emph{reachable} if there are $e_1, \dots,
    e_n \in \Tstates \cup \Ttransitions$ and $\info_0, \dots, \info_n \in
    \infospace[\src]$ such that $\info_0 = \Sinitinfo_{\src}$ and for all
    $0 < i \le n$
    \begin{itemize}
        \item if $e_i = s \in \Tstates$ then $\info_i =
            \Srefine_{\src}(\info_{i-1}, s)$,
        \item if $e_i = t \in \Ttransitions$ then $\info_i =
            \Supdate_{\src}(\info_{i-1}, t)$, and
        \item states $e_i \in \Tstates$ and the origins of transitions
            $e_i\in\Ttransitions$ are either $\Tinit$ or the target of a
            transition $e_j$ with $j < i$.
    \end{itemize}
\end{definition}

In our example above, $\info_0$ and $\info_u$ are reachable. For $\info_r$ the
sequence with $e_1 = B$ does not show that it is reachable because $B$ is not a
target of a transition earlier in the sequence. Information $\info_r$ could
still be reachable if the same information can be produced along a different
path, say first updating $\info_0$ along $\tuple{A, x, B}$ and then refining the
result on $B$. For the resulting information the sequence with $e_1 = \tuple{A,
x, B}$ and $e_2 = B$ shows reachability.

Dynamic heuristics are heuristics that depend on an information object in
addition to a state.

\begin{definition}
    A \emph{dynamic heuristic} over a transition system $\Tsystem$ with states
    $\Tstates$ depends on an information source $\src$ for $\Tsystem$ and
    is a function
    \[
        \hdyn\colon \Tstates \times \infospace[\src] \to \nonnegreals \cup
        \{\infty\}\text{.}
    \]
\end{definition}

In our landmark example, a dynamic heuristic could estimate the cost of each
state by the sum of the label cost of its known landmarks or by 0 if no landmarks are known.
For example, $\hlm(A, \info_0) = 3$ and $h(s, \info_0) =
0$ for all $s \in \{B,C,D\}$. After updating the information along $\tuple{A,y,C}$
the heuristic becomes more informed:
$\hlm(A, \info_1) = 3$, $\hlm(C, \info_1) = 1$, and $\hlm(B, \info_1) = \hlm(D,
\info_1) = 0$.

Static heuristics can be seen as the special case for a constant information
source, \ie{}, one where $\Supdate$ and $\Srefine$ do not modify the information.
Let us now define the dynamic counterparts of common static heuristic
properties.

\begin{definition}
    Let $\hdyn$ be a dynamic heuristic over a transition system with states
    $\Tstates$ and transitions $\Ttransitions$ and over information source
    $\src$.
    We say $\hdyn$ is
    \begin{description}[font=\normalfont]
        \item[\dynsafe] if $\hdyn(s, \info) = \infty$ implies $\hstar(s) =
            \infty$ for all $s \in \Tstates$ and all reachable $\info \in
            \infospace[\src]$;
        \item[\dynadmissible] if $\hdyn(s, \info) \leq \hstar(s)$ for all $s \in
            \Tstates$ and all reachable $\info \in \infospace[\src]$;
        \item[\dynconsistent] if $\hdyn(s, \info) \leq
            \Tcost(\ell) + \hdyn(s', \info)$ for all $\tuple{s,\ell,s'} \in \Ttransitions$ and all
            reachable $\info \in \infospace[\src]$; and
        \item[\dyngoalaware] if $\hdyn(s, \info) = 0$ for all $s \in
            \Tgoal$ and all reachable $\info \in \infospace[\src]$.
    \end{description}
\end{definition}

As in the static case, a \dyngoalaware{} and \dynconsistent{} heuristic is
\dynadmissible{}, and a \dynadmissible{} heuristic is \dynsafe{} and
\dyngoalaware{}. This can easily be shown by considering the static heuristics
that result from fixing the individual reachable $\info \in \infospace[\src]$.

If $\lmupdate$ and $\lmrefine$ guarantee that all stored labels are landmarks
for their states, then $\hlm$ is \dynadmissible\ (and thus \dynsafe\ and
\dyngoalaware). It is not \dynconsistent\ because for example
$\hlm(A,\info_0) = 3 > 2 + 0 = \Tcost(y) + \hlm(C,\info_0)$.

We also define a new property that describes that heuristic values can only
increase over time.

\begin{definition}
    A dynamic heuristic $\hdyn$ over a transition system with states $\Tstates$
    and transitions $\Ttransitions$ and over information source $\src$ is
    \emph{\dynmonotonic} if $\hdyn(s, \info) \leq \hdyn(s, \Supdate(\info, t))$
    and $\hdyn(s, \info) \leq \hdyn(s, \Srefine(\info, s'))$ for all reachable
    $\info \in \infospace[\src]$, all $s, s' \in \Tstates$, and all $t \in
    \Ttransitions$.
\end{definition}

If $\lmupdate$ and $\lmrefine$ guarantee that landmarks are never removed for
any state, then $\hlm$ is \dynmonotonic.

Monotonicity is particularly desirable for any \dynadmissible{} heuristic $h$ as
increasing values brings it closer to $\hstar$. Fortunately, we can easily
ensure monotonicity for $h$ by using $\hdyn'(s, \info) = \max(\hdyn(s, \info),
\hdyn(s, \info'))$ where $\info'$ is the information encountered so far that has
lead to the highest $h$-value for state $s$. Similar notions of monotonicity
have previously been described by \citet{eifler-fickert-socs2018} and
\citet{franco-torralba-socs2019}. Note that this monotonicity notion is
different from the monotonicity introduced by \citet{pohl-mi1977}, which is
equivalent to consistency \citep{pearl-1984}.

\subsection{Progression-Based Heuristics}

In our running example, information is only stored per state. While this is not
a general requirement of dynamic heuristics, it is an interesting special case
because it allows for simpler definitions, local reasoning, and, in practical
implementations, efficient storage of information. All our theoretical results
apply to dynamic heuristics in general but using this special case makes it more
natural to talk about several existing techniques like LTL trajectory
constraints~\cite{simon-roeger-socs2015} or landmark
progression~\citep{buechner-et-al-icaps2023}, which our running example is based
on.

We first define how information for a single state is modified on a low level
(Definition~\ref{def:progression-src}), and then use these transformations in
the definition of a special type of information source
(Definition~\ref{def:progression-based-info-src}).

\begin{definition}\label{def:progression-src}
    A \emph{progression source} $\inn$ consists of an information space
    $\infospace[\inn]$, and the following constant and functions:
    \begin{align*}
        \Sinitstateinfo_{\inn} &\in \infospace[\inn]\text{,}\\
        \Sprogress_{\inn}&\colon \infospace[\inn] \times \Ttransitions \to
            \infospace[\inn]\wand \\
        \Smerge_{\inn}&\colon \infospace[\inn] \times \infospace[\inn] \to
            \infospace[\inn]\text{.}
    \end{align*}
\end{definition}

For our running example we can define the progression source $\lmps$ where
$\lmpsspace = 2^{\{x, y\}}$ is the information space for a state and the initial
state $A$ is assigned $\lmpsinit = \{x, y\}$. The progression of a set of
landmarks $L$ for a state $s$ along a transition $t = \tuple{s, \ell, s'}$ are
the labels $\lmpsprogress(L, t) = L \setminus \{\ell\}$ because they are
guaranteed to be landmarks for $s'$ (while $\ell$ is not). If we get two sets of
landmarks $L_1, L_2$ for the same state, their union $\lmpsmerge(L_1, L_2) = L_1
\cup L_2$ is guaranteed to contain only landmarks for that state.

Beside our running example, we can also view two fundamental concepts in
search, namely \emph{$g$-values} and \emph{parent
pointers} of states, as a progression source.

\begin{definition}\label{def:parent-source}
    The \emph{parent source} $p$ is a progression source with $\infospace[p] =
    \nonnegreals \times (\Ttransitions \cup \{\myundef\})$ and
    \begin{align*}
        \Sinitstateinfo_p &= \tuple{0, \myundef}\text{,} \\
        \Sprogress_p(\tuple{g, t}, \tuple{s, \ell, s'}) &= \tuple{g + \Tcost(\ell), \tuple{s, \ell, s'}}\text{, and} \\
        \Smerge_p(\tuple{g, t}, \tuple{g', t'}) &=
            \begin{cases}
                \tuple{g, t} & \text{if } g \leq g' \\
                \tuple{g', t'} & \text{otherwise.}
            \end{cases}
    \end{align*}
\end{definition}

We now define information sources that update per-state information based on
this procedure of progressing and merging. This special case does not allow
refinement because we want it to be as restrictive as possible while still being
useful.

\begin{definition}\label{def:progression-based-info-src}
    Let $\inn$ be a progression source. The \emph{progression-based information
    source} $\src_\inn$ is an information source for a transition system with
    states $\Tstates$ and transitions $\Ttransitions$ where
    $\infospace[{\src_\inn}]$ is the set of partial functions  $\info\colon
    \Tstates \partialto \infospace[\inn]$, and
    \begin{align*}
        \Sinitinfo &= \{\Tinit \mapsto \Sinitstateinfo_{\inn}\}\\
        \Supdate(\info, t) &= \info_t
        \quad\text{for all $\info \in \infospace[\src_\inn]$ and $t \in \Ttransitions$}\\
        \Srefine(\info, s) &= \info
        \quad\text{for all $\info \in \infospace[\src_\inn]$ and $s \in \Tstates$}
    \end{align*}
    where $\info_t$ for a transition $t=\tuple{s, \ell, s'}$ is the same as
    $\info$ for all states other than $s'$ and maps $s'$ to
    \begin{align*}
        \begin{cases}
        \Sprogress(\info(s), t) \hspace{5em}\text{if $\info(s')$ undefined} \\
        \Smerge(\Sprogress(\info(s), t), \info(s')) \hspace{2.2em}\text{otherwise.}
        \end{cases}
    \end{align*}
\end{definition}

With this definition we get $\src_\lmps$, the progression-based information
source built on the progression source $\lmps$. Progressing the initial set of
landmarks $\{x,y\}$ along $\tuple{A, x, B}$ gives the landmarks $\{y\}$ for $B$.
Progressing this information along $\tuple{B, y, C}$ gives the empty set of
landmarks for $C$, but if we then also progress $\{x,y\}$ along $\tuple{A, y,
C}$ this gives new information $\{x\}$ for $C$ which has to be merged with the
existing information, so the information stored for $C$ after the progression is
$\emptyset \cup \{x\} = \{x\}$. The application section provides details on the
full landmark progression technique.

Similarly, we also get $\src_p$, a progression-based information source for
$g$-values and parent pointers. Note that for a reachable parent information
$\info$ with $\info(s) = \tuple{g, t}$, we can show by induction that following
parent pointers back to $\myundef$ yields a path to $s$ with cost of at most
$g$. This will be useful later.

Dynamic heuristics based on progression sources base the heuristic value of a
state on its stored information.

\begin{definition}
    A \emph{progression-based heuristic} is a dynamic heuristic with a
    progression-based information source $\src_\inn$ where $\hdyn(s, \info) = 0$
    if $\info(s)$ is undefined, and where $\info(s) = \info'(s)$ implies
    $\hdyn(s,\info) = h(s,\info')$ for all pairs $\info, \info' \in
    \infospace[{\src_\inn}]$ and all states $s \in \Tstates$.
\end{definition}

Because of the way progression-based information sources are updated,
progression-based heuristics are always $0$ for states not yet discovered in the
search, \ie{}, states $s$ where $\info(s)$ is undefined. These default values
can easily violate \dynconsistency{}
%
%
and we therefore consider a heuristic $\hdyn$ to be \emph{partially
\dynconsistent} if $\hdyn(s,\info) \leq \Tcost(\pi) + \hdyn(s',\info)$ for all
paths $\pi$ from $s$ to $s'$ for which $\info(s)$ and $\info(s')$ are defined.

\section{Dynamic Heuristic Search Framework}

A typical approach to finding solutions in a transition system is to explore it
sequentially starting from the initial state $\Tinit$. Algorithm~\ref{alg:main}
implements such a forward search and maintains information during the
exploration. Besides the search direction, Algorithm~\ref{alg:main} leaves
choices such as the expansion order open. This captures a wide range of
instantiations that may vary in how often they use, update or refine the
information within the search.

\begin{algorithm}[t]
\begin{algorithmic}[1]
\small
    \REQUIRE transition system $\tuple{\Tstates, \Tlabels, \Tcost,
    \Ttransitions, \Tinit, \Tgoal}$; set of information sources $\srcs$
    \ENSURE \solvable{} or \unsolvable{}
    \STATE $\infos \assign \{\src \mapsto \Sinitinfo_{\src} \mid \src \in \srcs\}$
        \label{lin:main-init-infos}
    \STATE $\Sknown \assign \{\Tinit\}$\label{lin:main-init-sknown}
    \LOOP
        \STATE choose applicable operation\label{lin:choose-op}
        \SWITCH{operation}
        \CASE{\genunknown}\label{lin:genunknown}
            \STATE choose $t = \tuple{s, \ell, s'} \in
                \Ttransitions$ such that $s \in \Sknown$ and\\
                \quad$s' \notin \Sknown$\label{lin:choose-t}
            \FORALL{$\src \in \srcs$}\label{lin:unknown-update-loop}
                \STATE $\infos(\src) \assign \Supdate_{\src}(\infos(\src), t)$
                \label{lin:unknown-update}
            \ENDFOR
            \STATE $\Sknown \assign \Sknown \cup \{s'\}$
        \ENDCASE
        \CASE{\genknown} \label{lin:genknown}
            \STATE choose $t = \tuple{s, \ell, s'} \in \Ttransitions$ such that
                $s,s' \in \Sknown$
            \FORALL{$\src \in \srcs$}\label{lin:known-update-loop}
                \STATE $\infos(\src) \assign \Supdate_{\src}(\infos(\src), t)$
                \label{lin:known-update}
            \ENDFOR
        \ENDCASE
        \CASE{\refine} \label{lin:refine}
            \STATE choose $s \in \Sknown$
            \FORALL{$\src \in \srcs$}
                \STATE $\infos(\src) \assign \Srefine_{\src}(\infos(\src), s)$
            \ENDFOR \label{lin:end-refine}
        \ENDCASE
        \CASE{\declsolvable} \label{lin:decsolvable}
            \STATE ensure that $\Sknown$ contains a goal
                state\label{lin:solvable-cond}
            \RETURN \solvable{}\label{lin:main-solvable}
        \ENDCASE
        \CASE{\declunsolvable} \label{lin:decunsolvable}
            \STATE ensure that \Sknown{} does not contain a goal
                state\label{lin:goal-not-known}
            \STATE ensure that there is no $\tuple{s, \ell, s'} \in \Ttransitions$
                such that\\
                \quad$s \in \Sknown$ and $s' \notin \Sknown$\label{lin:no-trans}
            \RETURN \unsolvable{} \label{lin:main-unsolvable}
        \ENDCASE
    \ENDLOOP
\end{algorithmic}
\caption{Dynamic Heuristic Search Framework}%
\label{alg:main}
\end{algorithm}

The algorithm maintains one information object $\infos[\src]$ for each source
$\src$, initializes it at the start, updates it when exploring transitions, and
optionally refines it for known states. By keeping track of known states in
$\Sknown$ it guarantees that exactly the reachable information according to
Definition~\ref{def:reachable-information} can occur in the algorithm. If at
some point a goal state is known, the problem is solvable. If not, and no new
states can be reached, then it is unsolvable. A simple induction shows that
$\Sknown$ contains reachable states.

\begin{lemma}\label{lem:reachable}
    Every state in $\Sknown$ in Algorithm~\ref{alg:main} is reachable.
\end{lemma}

We use this result to show that Algorithm~\ref{alg:main} is sound.

\begin{theorem}\label{thm:soundness}
    Any instantiation $\mathcal{A}$ of Algorithm~\ref{alg:main} is \emph{sound}.
\end{theorem}
\begin{proof}
    Instantiation $\mathcal{A}$ only returns \solvable\ if $\Sknown$ contains a goal state $\goal$.
    Then Lemma~\ref{lem:reachable} implies that $\goal$
    is reachable and the underlying problem is solvable.

    Further, $\mathcal{A}$ only returns \unsolvable\ if \genunknown{} is not
    applicable (line~\ref{lin:no-trans}) and thus that
    $\Sknown$ contains all reachable states. Since no goal state is in $\Sknown$
    (line~\ref{lin:goal-not-known}), this means no goal is reachable.
\end{proof}

We would also like instantiations of Algorithm~\ref{alg:main} to terminate in a
finite number of steps.

\begin{theorem}\label{thm:completeness}
    An instantiation $\mathcal{A}$ of Algorithm~\ref{alg:main} is
    \emph{complete} if all loops over \genknown{} and \refine{}
    are finite.
\end{theorem}
\begin{proof}
%

    Whenever $\mathcal{A}$ chooses an operation, either \genunknown{},
    \declsolvable{}, or \declunsolvable{} is applicable: if a transition
    $\tuple{s, \ell, s'} \in \Ttransitions$ with $s \in \Sknown$ and $s' \notin
    \Sknown$ exists, then \genunknown{} is applicable, otherwise either
    \declsolvable{} or \declunsolvable{} is applicable depending on whether a
    goal state is contained in $\Sknown$ or not.

    Further, the number of times operation \genunknown{} can be applied is
    bounded by $\card{\Tstates}$ as it requires $s' \notin \Sknown$ and adds
    $s'$ to $\Sknown$ when applied.

    Since no infinite loops over \genknown{} and \refine{} occur in
    $\mathcal{A}$, and because one of the remaining operations must be
    applicable, $\mathcal{A}$ must eventually terminate with \declsolvable{} or
    \declunsolvable.
\end{proof}

\section{Dynamic \boldmath$\astar$}

Algorithm~\ref{alg:dynastar} instantiates the dynamic heuristic search framework
as $\dynastar$, a generalization of $\astar$ for dynamic heuristics that, as we
will see, covers many existing techniques.

\begin{algorithm}[t!]
    \begin{algorithmic}[1]
    \small
        \REQUIRE transition system $\tuple{\Tstates, \Tlabels, \Tcost,
        \Ttransitions, \Tinit, \Tgoal}$; set of $\srcs = \{\src_p, \src_h\}$;
            dynamic heuristic $\hdyn$ over $\src_h$; flag \reeval{}
        \ENSURE solution or \unsolvable{}
        \STATE $\infos \assign \{\src \mapsto \Sinitinfo_{\src} \mid \src \in
            \srcs\}$\label{lin:dynastar-init-infos}
        \STATE $i \assign 0$; $j \assign 0$\COMMENT{counters for iterations and
            steps within them}
        \STATE $\Sknown \assign \{\Tinit\}$; $\closed \assign \{\}$
        \STATE $\open \assign$ \NEW{} priority queue with elements\\
            \quad$\tuple{\text{state}, \gval, \hval}$, prioritizing lower $f =
            \gval + \hval$
        \IF{$h^{i, j}(\Tinit) < \infty$}\label{lin:init-dead-end-check}
            \STATE insert $\tuple{\Tinit, g^{i, j}(\Tinit), h^{i, j}(\Tinit)}$
                into \open{}\label{lin:insert-init}
        \ENDIF
        \WHILE{\open{} is not empty}
            \STATE $i \assign i+1$; $j \assign 0$\label{lin:i-inc-j-reset}
            \STATE $\tuple{s, \gval, \hval} \assign$ pop highest priority
            element from $\open$\label{lin:pop}\label{lin:while-loop-start}
            \IF{$s \in \closed$}\label{lin:closed-check}
                \CONTINUE
            \ENDIF
            \FORALL{$\src \in \srcs$}\label{lin:refine-loop}
                \STATE $\infos(\src) \assign
                    \Srefine_{\src}(\infos(\src), s)$\label{lin:end-refine-loop}
            \ENDFOR
            \STATE $j \assign 1$
            \IF{\reeval{} \AND $\hval < h^{i,j}(s)$}\label{lin:reeval-check}
                \IF{$h^{i,j}(s) < \infty$}\label{lin:reeval-dead-end-check}
                    \STATE insert $\tuple{s, g^{i, j}(s), h^{i, j}(s)}$ into
                        \open{}\label{lin:reeval-insert}
                \ENDIF
                \CONTINUE\label{lin:reeval-continue}
            \ENDIF
            \STATE $\closed \assign \closed \cup \{s\}$\label{lin:close-s}
            \IF{$s$ is a goal state}\label{lin:goal-check}
                \RETURN extracted solution according to $\src_p$
                    \label{lin:dynastar-solvable}
            \ENDIF
            \FORALL{$t = \tuple{s, \ell, s'} \in \Ttransitions$}\label{lin:for-successors}
                \IF{$s' \in \Sknown$}\label{lin:new-check}
                    \STATE $\oldg \assign g^{i, j}(s')$
                \ELSE
                    \STATE $\oldg \assign \text{undefined}$
                \ENDIF
                \FORALL{$\src \in \srcs$}\label{lin:update-loop}
                    \STATE $\infos(\src) \assign \Supdate_{\src}(\infos(\src),t)$
                    \label{lin:update}
                \ENDFOR
                \STATE $j \assign j + 1$\label{lin:j-inc}
                \STATE $\Sknown \assign \Sknown \cup \{s'\}$\label{lin:insert-sknown}
                \IF{$h^{i, j}(s') = \infty$}\label{lin:infty-check}
                    \CONTINUE\label{lin:infty-continue}
                \ENDIF
                \IF[first path to $s'$]{$\oldg = \text{undefined}$}\label{lin:insert-new-check}
                    \STATE insert $\tuple{s', g^{i, j}(s'), h^{i, j}(s')}$ into \open{}
                        \label{lin:insert-new}
                \ELSIF[cheaper path to $s'$]{$\oldg > g^{i, j}(s')$}\label{lin:insert-cheaper-check}
                    \IF{$s' \in \closed$}\label{lin:re-expand-check}
                        \STATE $\closed \assign \closed \setminus \{s'\}$\label{lin:reopen}
                    \ENDIF
                    \STATE insert $\tuple{s', g^{i, j}(s'), h^{i, j}(s')}$ into \open{}\label{lin:insert-cheaper}
                \ENDIF
            \ENDFOR
        \ENDWHILE
        \RETURN \unsolvable{}\label{lin:dynastar-unsolvable}
    \end{algorithmic}
    \caption{$\dynastar$ with optional re-evaluation.}%
    \label{alg:dynastar}%
\end{algorithm}

The algorithm depends on the progression-based information source $\src_p$,
\ie{} the parent source $p$, to track $g$-values and parent pointers as well as
on an information source $\src_h$ used by a dynamic heuristic $\hdyn$. Like
regular $\astar$, $\dynastar$ maintains \open{}, a priority queue of states
ordered by $f=g+h$, and \closed{}, a hash set of states.

Since information changes over time, we use the counters $i$ for iterations and
$j$ for steps within the iteration, to reference specific \emph{times} $t =
\tuple{i, j}$. We then refer to information available to the algorithm at a
particular time $t$ by $\infos^t$, which is unambiguous since $i$ or $j$ always
change immediately after $\infos$ is modified. Similarly, we use $g^t(s)$ to
refer to the $g$-value stored in $\infos^t(\src_p)$ for state $s$, and
$\hdyn^t(s)$ to refer to $\hdyn(s, \infos^t(\src_h))$.

For example, the initial state $\Tinit$ is inserted into \open{} in
line~\ref{lin:insert-init} with $g^{0,0}(\Tinit)$ and $h^{0,0}(\Tinit)$, then
popped in line~\ref{lin:pop} in the first iteration at time $\tuple{1,0}$. If
its first successor $s'$ is not pruned, it is inserted in
line~\ref{lin:insert-new} with $g^{1,2}(s')$ and $h^{1,2}(s')$. We say
$\tuple{i,j} < \tuple{i', j'}$ if $i < i'$ or if $i = i'$ and $j < j'$.

Entries in \open{} are sorted by their $f$-value at time of insertion, not the
current $f$-values of the contained states. Put differently, if a heuristic
improves during the search, this does not directly affect entries already on
\open{}.

$\dynastar$ implements \emph{delayed duplicate detection}, a technique often
used in practice, \eg{} in Fast Downward \citep{helmert-jair2006}, that allows
multiple entries for a state on \open{} and later eliminates duplicates in
line~\ref{lin:closed-check}. The alternative is to detect duplicate states early
(at the time of insertion) and update open list entries when a cheaper path is
found. However, the optimal time complexity of this version relies on an
efficient \textit{decrease-key} operation which practical implementations of
priority queues like binary heaps do not support.

Optionally, states whose heuristic value improved since being inserted into
\open{} can be \emph{re-inserted} with their new heuristic value in
line~\ref{lin:reeval-insert}, an idea used by systems relying on dynamic
heuristics \citep[\eg{},][]{karpas-domshlak-ijcai2009, zhang-bacchus-aaai2012,
eifler-fickert-socs2018}. We call this modification \emph{re-evaluation} and
study the algorithm with and without it. States that are neither detected as
duplicates nor re-evaluated are \emph{expanded}. If a goal state is expanded, a
plan is constructed from the parent pointers stored in $\src_p$, otherwise all
successors are considered.

Successors with a heuristic value of $\infty$ are skipped
(line~\ref{lin:infty-continue}) and unknown successors are inserted into \open{}
(line~\ref{lin:insert-new}). If we find a new path to a known successor
$s'$, then $s'$ is only placed on \open{} if this new path is cheaper than the
one we already knew (line~\ref{lin:insert-cheaper-check}). In case $s'$ was
already closed, it is \emph{reopened} (line~\ref{lin:reopen}); We will later
show that, with some restrictions on the heuristic, states are never reopened.

$\dynastar$ simulates $\astar$ if heuristic $h$ is static, \ie{}, $h(s, \info)$
is independent of $\info$. Note that the value of $\reeval$ has no effect on the
algorithm in this case: re-evaluation never occurs because $\hval = h^{i,j}(s)$
in line~\ref{lin:reeval-check}.

Let us now connect $\dynastar$ to the generic framework.

\begin{theorem}\label{thm:dynastar-instantiation}
    The algorithm $\dynastar$ using a \dynsafe{} dynamic heuristic is an
    instantiation of the dynamic heuristic search framework.
\end{theorem}
\begin{proof}
    We first show that $\dynastar$ only modifies $\infos$ and $\Sknown$ in ways
    allowed by the framework. The initialization in
    line~\ref{lin:dynastar-init-infos} is identical. The loop in
    lines~\ref{lin:refine-loop}--\ref{lin:end-refine-loop} corresponds to a
    \refine{} operation. We know $s\in \Sknown$ at this time because only known
    states are added to \open{}. The loop in
    lines~\ref{lin:update-loop}--\ref{lin:update} corresponds to \genknown{} if
    $s' \in \Sknown$ in line~\ref{lin:new-check}, and to \genunknown{}
    otherwise, in which case $s'$ is inserted into $\Sknown$ in
    line~\ref{lin:insert-sknown}.

    Next, we have to show that $\dynastar$ only terminates in case the framework
    can terminate. Line~\ref{lin:dynastar-solvable} corresponds to
    \declsolvable\ which is applicable due to the condition in
    line~\ref{lin:goal-check} and because $s$ was inserted into $\Sknown$ before
    inserted into $\open$ (and now popped from there).

    Lastly, in cases where line~\ref{lin:dynastar-unsolvable} returns
    \unsolvable{}, we have to show that
    \begin{enumerate*}
        \item no goal state is in $\Sknown$ and that\label{cas:no-goal-known}
        \item there is no transition leaving $\Sknown$.\label{cas:no-transition}
    \end{enumerate*}

    Condition~\ref{cas:no-goal-known} holds because
    every state in $\Sknown$ was inserted into \open{} at least once. Given that
    \open{} is empty in line~\ref{lin:dynastar-unsolvable}, all entries have
    been popped in line~\ref{lin:pop}. If there were a goal state in $\Sknown$,
    a \dynsafe\ heuristic assigns the state a finite heuristic value and
    popping it would either re-evaluate it (putting it back on \open{} without
    closing it) or terminate the search with a solution.

    Condition~\ref{cas:no-transition} does not necessarily hold since
    $\dynastar$ can prune states with infinite heuristic values in
    lines~\ref{lin:init-dead-end-check}, \ref{lin:reeval-dead-end-check},
    and~\ref{lin:infty-check}. Executing line~\ref{lin:dynastar-unsolvable} does
    thus not directly correspond to \declunsolvable{} in the framework. Instead,
    we show that it corresponds to repeated uses of \genunknown{} followed by
    \declunsolvable{}. Since the heuristic is \dynsafe, no goal state is
    reachable from all pruned states. We can thus use \genunknown{} to fully
    explore the state space reachable from $s$ and will not add a goal state to
    $\Sknown$. For all states that were not pruned, we can see with the same
    argument as above that they were added to \open{} and eventually expanded,
    adding all their unpruned successors to $\Sknown$.
\end{proof}

$\dynastar$ using a \dynsafe\ heuristic is sound due to
Theorems~\ref{thm:soundness} and \ref{thm:dynastar-instantiation} and we can
also show that it can be complete.

\begin{theorem}
    $\dynastar$ is complete if chains of \refine{} operations converge, \ie{},
    for all reachable $\info_0$, states $s$, and $\info_i =
    \refine_{\src_h}(\info_{i-1},s)$ with $i > 0$, there is an $n \geq 0$ such
    that $\hdyn(s,\info_n) = \hdyn(s,\info_{n+1})$.
\end{theorem}
\begin{proof}
    Using Theorem~\ref{thm:completeness}, we only need to show that there are no
    infinite loops of \refine{} and \genknown{} operations. The former cannot
    happen because \refine{} operations converge. The latter holds because known
    transitions are only considered if we find a cheaper path and after
    generating a first path to a state there are only finitely many values that
    can occur as cheaper path costs.
\end{proof}

\subsection{Optimality of Solutions}

We first show that $\dynastar$ returns optimal solutions when using a
\dynadmissible{} heuristic. The proof generally follows the same line of
reasoning as the one for $\astar$ with a static admissible heuristic. Our
situation is more complicated because using delayed duplicate detection with
dynamic heuristics means that there can be states on \open{} where both
$g$-values and $h$-values are outdated because we discovered both a cheaper path
to and a higher heuristic value for those states before popping them from the
open list. Additionally, we want to show optimality both with and without
re-evaluation which adds an additional source of changes to the open list.

As in the static case \citep[\eg{},][]{hart-et-al-ieeessc1968},
our general strategy is to show that at any
time before termination, there is an entry on \open{} that represents a prefix
of an optimal plan, and the algorithm can only terminate by completing one of
these prefixes. To state this formally, we first define some auxiliary notation.


\begin{definition}
    A state $s$ is called \emph{settled} in iteration $i$ if it was expanded in
    some iteration $i' < i$ with $g^{i'\!,0}(s) = g^*(s)$.
\end{definition}

Once a state is settled, we know a cheapest path to it and have seen its
successors along this path. If we consider an optimal path, some prefix of it
will be settled and the next state along this path will be on \open{} with an
optimal $g$-value.

\begin{lemma}\label{lem:optimal-continuation}
    Consider $\dynastar$ with a \dynsafe{} dynamic heuristic in iteration $i$.
    Let $s$ be a state settled in iteration~$i$ and let $s'$ be a state where
    $\hstar(s') < \infty$ such that $s$ is a predecessor of $s'$ and there is an
    optimal path $\tuple{t_1, \ldots, t_n}$ to $s'$ with $t_n =
    \tuple{s,\ell,s'}$, \ie{}, $s'$ can be reached optimally through $s$.
    Then \open{} contains an entry $\tuple{s', \gval, \cdot}$ in iteration $i$
    and $\gval =  g^*(s')$, or $s'$ is settled in iteration $i$.
\end{lemma}
\begin{proof}
    Consider states $s$ and $s'$ such that the conditions of the lemma hold in
    iteration $A$. Because $s$ is settled in iteration $A$, it must have
    been expanded in an iteration $B < A$ with $g^{B,0}(s) = g^*(s)$.

    Because $s$ was expanded, iteration $B$ reached line~\ref{lin:close-s}.
    Also, $\dynastar$ did not terminate in line~\ref{lin:dynastar-solvable} in
    iteration $B$, because otherwise iteration $A$ would not exist. Thus
    $\tuple{s, \ell, s'}$ was considered at a time $\tuple{B,j}$, we found an
    optimal path to $s'$, and $g^{B,j}(s') = g^*(s')$.

    Let $\tuple{C,k}$ be the first time that the $g$-value of $s'$ is reduced to
    $g^*(s')$. (This happens at $\tuple{B, j}$ at the latest but could also have
    happened earlier.) At this time, either $s'$ was not known or a cheaper path
    to it was found. Because $\hstar(s') < \infty$ and $\hdyn$ is \dynsafe{}, we
    have $\hdyn^{C,k}(s') < \infty$ and $s'$ was inserted into \open{} with a
    $g$-value of $g^{C,k}(s') = g^*(s')$. Any entry for $s'$ added to \open{}
    after $\tuple{C,k}$ also has a $g$-value of $g^*(s')$.

    If the entry $\tuple{s', g^*(s'), \cdot}$ inserted in iteration $C$ is still
    on \open{} in iteration $A$, then the first alternative of the lemma's consequent
    is satisfied.
    Otherwise, one or more entries for $s'$ were popped from \open{} after $C$
    and
    \begin{enumerate*}
        \item\label{case:reeval} re-evaluated,
        \item\label{case:ignore} ignored because $s'$ was in \closed{}, or
        \item\label{case:expand} expanded.
    \end{enumerate*}
    In case~\ref{case:reeval}, a new entry $\tuple{s', g^*(s'), \cdot}$ is added
    to \open{} and the argument restarts at the beginning of this paragraph,
    replacing $C$ with the iteration of this reinsertion.
    Case~\ref{case:ignore} can only happen if $s'$ was expanded after $C$
    because Algorithm~\ref{alg:dynastar} ensures that states inserted into
    $\open$ are not closed\footnote{%
        This is explicit except in line~\ref{lin:insert-new} where $s'$ cannot
        be closed because it was just discovered.}
    and only expanding $s'$ can close it. It follows that $s'$ was expanded
    (case~\ref{case:expand}) with $g$-value $g^*(s')$ in an iteration between
    $C$ and $A$ and thus that $s'$ is settled in $A$, satisfying the second
    alternative of the lemma's consequent.
\end{proof}

With \dynadmissible{} heuristics, this result implies that there always is an
entry with an $f$-value of at most the optimal solution cost.

\begin{lemma}
    \label{lem:f-bound}
    Consider $\dynastar$ with a \dynadmissible{} dynamic heuristic $\hdyn$ for a
    solvable transition system with initial state $\Tinit$. Then \open{}
    contains an entry $\tuple{\cdot, \gval, \hval}$ with $\gval + \hval \leq
    \hstar(\Tinit)$ at the beginning of each iteration.
\end{lemma}
\begin{proof}
    Consider an optimal solution $\pi = \tuple{t_1, \dots, t_n}$ through states
    $s_0, \dots, s_n$, \ie{}, $t_i = \tuple{s_{i-1}, \cdot, s_i}$. Let $s_k$ be
    the first state that is not settled in an iteration $i$. Note that $s_n$
    could not have been settled without $\dynastar$ terminating.

    In case $k = 0$, we know that $s_0$ is settled in iteration~$1$, so $i$ must
    be $0$. Since $\hdyn$ is \dynadmissible\ and hence $\hstar(\Tinit) <
    \infty$, the entry $\tuple{\Tinit, 0, \hdyn^{0,0}(\Tinit)}$ was inserted
    into $\open$ and is the only entry at the beginning of iteration~$i$.
    The lemma then follows from the admissibility of $\hdyn^{0,0}$.

    In case $k > 0$, we know that $s_{k-1}$ is settled in iteration~$i$.
    Moreover, $\tuple{t_{k+1}, \ldots, t_n}$ is a path from $s_k$ to a goal
    state showing that $\hstar(s_k) < \infty$ and $\tuple{t_1, \ldots, t_k}$ is
    an optimal path to $s_k$. We can therefore use
    Lemma~\ref{lem:optimal-continuation} in iteration $i$ with $s = s_{k-1}$ and
    $s' = s_k$. Because $s_k$ is not settled in iteration $i$, we know that
    there must be an entry $\tuple{s_k, \gval, \hval}$ on \open{} where $\gval =
    g^*(s_k)$ and $\hval = \hdyn^{\tins}(s_k)$ for some time $\tins$. With the
    admissibility of $\hdyn^{\tins}$, we get $\gval + \hval = \gstar(s_k) +
    \hdyn^{\tins}(s_k) \leq \gstar(s_k) + \hstar(s_k) = \hstar(\Tinit)$.
\end{proof}

In the context of graphs with unknown obstacles, \citet{koyfman-et-al-aamas2024}
show that a variant of $\astar$ returns optimal solutions when using a
dynamic heuristic that is PDA\@. As PDA requires that a prefix of an optimal
solution is closed and its continuation exists on \open{},
Lemma~\ref{lem:f-bound} directly follows. Showing that a heuristic is PDA thus
requires reasoning about \open{} and depends on the search behavior.
Our Lemmas~\ref{lem:optimal-continuation} and~\ref{lem:f-bound} only depend on
the heuristic being \dynadmissible{}, which is independent of the search
algorithm.

While there is an entry with $f$-value of at most $\hstar(\Tinit)$ on $\open$,
$\dynastar$ cannot yield a solution of cost $c > \hstar(\Tinit)$.

\begin{theorem}\label{the:dynastar-optimality}
    $\dynastar$ with a \dynadmissible{} dynamic heuristic returns optimal
    solutions.
\end{theorem}
\begin{proof}
    Assume the algorithm terminated with a solution of cost $c > \hstar(\Tinit)$
    after popping $\tuple{s, g^\tins(s), \hdyn^\tins(s)}$ from \open{}
    at time $\tpop$. Then $g^{\tins}(s) + \hdyn^{\tins}(s) = g^{\tins}(s) + 0
    \ge g^\tpop(s) = c > \hstar(\Tinit)$. This contradicts
    Lemma~\ref{lem:f-bound} because $\dynastar$ pops the
    entry in \open{} with minimal $f$-value.
\end{proof}

The results in this chapter also apply to $\astar$ as a special case.
A corollary is that $\astar$ with an admissible static heuristic is also optimal
when using delayed duplicate detection.

\subsection{Reopening}

A classic result for $\astar$ is that states are never reopened when the
heuristic is consistent. One way to prove this is by showing that the sequence
of $f$-values of popped open list entries is monotonically increasing. If we
then consider the projection of this sequence to a single state, we can use the
fact that the heuristic value $h(s)$ of a static heuristic is constant to see
that multiple copies of a state are expanded in order of increasing $g$-value.
If the lowest $g$-value of a state is expanded first, the state will never be
reopened.

Interestingly, we can show that the intermediate result of monotonically
increasing $f$-values also holds for dynamic heuristics with the right
properties, but reopening may still occur. Since the intermediate result is not
useful in this setting, we only present the full proof in this extended version
and provide a sketch in the published paper \citep{christen-et-al-icaps2025}.

We start with an auxiliary result showing that the $f$-value of a state cannot
decrease between the time it is inserted into \open{} and the time it is popped.

\begin{lemma}\label{lem:f-value-greater-at-pop}
    Let $\tuple{s, \gval, \hval}$ be an element popped at time $\tpop$ by
    $\dynastar$ with a \dynmonotonic\ dynamic heuristic. If $s \notin \closed{}$
    at time $\tpop$, then $\gval + \hval \le g^\tpop(s) + h^\tpop(s)$.
\end{lemma}
\begin{proof}
    Let $\tins$ be the last time before $\tpop$ when an entry $\tuple{s, \gval,
    \hval}$ was added to \open{}, i.e., $\gval = g^\tins(s)$ and $\hval =
    h^\tins(s)$.

    If $g^\tins(s) = g^\tpop(s)$ then the statement follows from
    \dynmonotonicity\ of $h$, otherwise the $g$-value of $s$ was decreased to
    $g^\tpop(s)$ at some time $\tupd$ with $\tins < \tupd < \tpop$. If
    $h^\tupd(s) = \infty$ then $h^\tpop(s) = \infty$, satisfying the lemma.
    Otherwise, $\tuple{s, g^\tupd(s), h^{\tupd}(s)}$ is inserted into \open{}.

    If $g^\tins(s) + h^\tins(s) \le g^\tupd(s) + h^\tupd(s)$, then the result
    follows from $g^\tupd(s) = g^\tpop(s)$ and the monotonicity of $h$.
    Otherwise, $\tuple{s, g^\tupd(s), h^{\tupd}(s)}$ can no longer be
    on \open{} at time $\tpop$ because of its lower $f$-value.

    However, since $s \notin \closed$ at times $\tupd$ and $\tpop$, and because
    $s$ has the same $g$-value at these times, it could not have been added to
    closed and reopened in this interval. The only other reason why $\tuple{s,
    g^\tupd(s), h^{\tupd}(s)}$ is no longer on \open{} is that $s$ was
    re-evaluated at time $\tre$ and a resulting entry $\tuple{s, g^\tre(s),
    h^{\tre}(s)}$ is still on \open{} at time $\tpop$. For this entry we have
    $g^{\tre}(s) = g^{\tpop}(s)$ because $\tupd < \tre < \tpop$ and $h^{\tre}(s)
    \le h^{\tpop}(s)$ because $\tre < \tpop$, and $h$ is monotonic.

    At time $\tpop$, we then have $g^{\tins}(s) + h^{\tins}(s) \le g^{\tre}(s) +
    h^{\tre}(s) \le g^{\tpop}(s) + h^{\tpop}(s)$.
\end{proof}

We can use Lemma~\ref{lem:f-value-greater-at-pop} to show that $f$-values popped
by $\dynastar$ are non-decreasing with similar restrictions to the heuristic as
in the static case.

\begin{theorem}\label{the:monotonic-f}
    The sequence of $f$-values popped by $\dynastar$ from $\open$ is
    non-decreasing if it uses a \dynadmissible{}, \dynmonotonic{}, and
    \dynconsistent{} dynamic heuristic.
\end{theorem}
\begin{proof}
    We show that any entry inserted into \open{} has an $f$-value at least as
    high as the value that was popped in the same iteration. Since the popped
    element had the minimal $f$-value at the time, this is sufficient to show
    that future pops cannot have lower $f$-values.

    For re-evaluations, Lemma~\ref{lem:f-value-greater-at-pop} shows what we
    want. For insertion of successor states, we consider three time steps $\tins
    < \tpop < \tsuc$. At time $\tpop$ an element $\tuple{s, \gval, \hval}$ is
    popped from \open{} (Algorithm~2, line~9 \citep{christen-et-al-icaps2025})
    which was inserted earlier at time $\tins$ (i.e., $\gval = g^\tins(s)$ and
    $\hval = h^\tins(s)$). Time $\tsuc$ is in the same iteration as $\tpop$
    while expanding $s$ and inserting entry $\tuple{s', g^\tsuc(s'),
    h^\tsuc(s')}$ for a successor $\tuple{s, \ell, s'}$ of $s$.
    \begin{align}
        g^\tins(s) + h^\tins(s)
        &\le g^\tpop(s) + h^\tpop(s)   \label{eq:monotonic-f:lemma}\\
        &\le g^\tsuc(s) + h^\tpop(s) \label{eq:monotonic-f:g-unchanged}\\
        &\le g^\tsuc(s) + h^\tsuc(s) \label{eq:monotonic-f:monotonicity}\\
        &\le g^\tsuc(s) + \Tcost(\ell) + h^\tsuc(s') \label{eq:monotonic-f:consistency}\\
        &= g^\tsuc(s') + h^\tsuc(s') \label{eq:monotonic-f:cheapest-path}
    \end{align}
    Step~\eqref{eq:monotonic-f:lemma} follows from
    Lemma~\ref{lem:f-value-greater-at-pop};
    step~\eqref{eq:monotonic-f:g-unchanged} is true because the $g$-value of $s$
    is not changed while expanding $s$;
    step~\eqref{eq:monotonic-f:monotonicity} holds because of
    \dynmonotonicity{}; step~\eqref{eq:monotonic-f:consistency} holds because of
    \dynconsistency{}; and step~\eqref{eq:monotonic-f:cheapest-path} holds
    because we only insert an entry for $\tuple{s, \ell, s'}$ if this transition
    provides the best known path to $s'$ at the time.
\end{proof}
Note that Theorem~\ref{the:monotonic-f} requires consistency only between known
states, thus it also holds for partially \dynconsistent{} heuristics.

\begin{figure}
  \centering
  \begin{tikzpicture}
    \tikzset{%
        y=1.6cm,
        x=1.05cm,
        state/.style={state common, text width=9mm},
    }
    \node[state] (a) at (0, 0) {$A$\\{\small $1$}};
    \node[state] (b) at (2, 0) {$B$\\{\small $0 \to 1$}};
    \node[state] (c) at (.8, -1) {$C$\\{\small $0 \to 1$}};
    \node[state] (d) at (4, 0) {$D$\\{\small $0 \to 3$}};
    \node[state] (e) at (3.2, -1) {$E$\\{\small $0 \to 4$}};
    \node[state, goal] (f) at (6, 0) {$F$\\{\small $0$}};

    \draw[trans] (-1.05, 0) to (a);
    \draw[emph trans] (a) to node[trans label] {\small 1} (b);
    \draw[emph trans] (a) to node[trans label, inner sep=1pt] {\small 2} (c);
    \draw[trans, to goal] (a) to[bend left=25] node[trans label] {\small 8} (f);
    \draw[trans] (a) to[bend right=30] node[trans label] {\small 6} (d);
    \draw[emph trans] (b) to node[trans label] {\small 5} (d);
    \draw[emph trans] (c) to node[trans label] {\small 1} (e);
    \draw[trans, to goal] (d) to node[trans label] {\small 3} (f);
    \draw[trans] (e) to node[trans label, inner sep=1pt] {\small 1} (d);
\end{tikzpicture}
  \caption{Example state space, showing that reopening can occur with
    \dynconsistent{} heuristics. Edge costs are written on the edges, heuristic
    values are written inside the states. The heuristic is dynamic and the
    heuristic value of a state with label $x \to y $ changes from $x$ to $y$
    once it is reached along the bold incoming edge.}
  \label{fig:reopening-counterexample}
\end{figure}
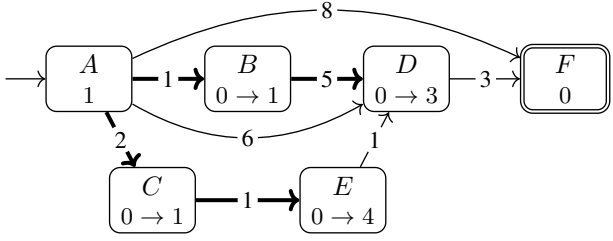

Without re-evaluation $\dynastar$ may reopen states even if the heuristic
satisfies the conditions of Theorem~\ref{the:monotonic-f}. In the example in
Figure~\ref{fig:reopening-counterexample}, states are expanded in the order $A,
B, C, D, E, F$. When $A$ is expanded, $F$ is added to the open list with the
suboptimal $f$-value of $8+0$ and $D$ is added with a suboptimal value of $6+0$.
While $B$ and $C$ are added to the open list, their heuristic values increase from
$0$ to $1$. Next $B$ is expanded and discovers a new path to $D$ along which we
find information to improve the heuristic value of $D$ to $3$. This path is not
cheaper so $D$ remains on \open{} with a value of $6+0$. After expanding $C$ and
increasing the heuristic value of $E$, \open{} is $\left\langle\tuple{D, 6, 0},
\tuple{E, 3, 4}, \tuple{F, 8, 0}\right\rangle$. State $D$ is expanded first but
adds nothing to $\open$ because $F$ is reached through a more expensive path
than the one found previously. Then $E$ is expanded and finds the optimal path
to $D$. If we do not reopen $D$ at this point, the algorithm terminates with the
suboptimal path $\tuple{\tuple{A, \cdot, F}}$.

Note that $f$-values of popped states
increase monotonically, the heuristic is monotonically increasing, admissible
and consistent throughout the search, and only updated along transitions (i.e.,
progression-based).

At the time $D$ was expanded, its heuristic value was higher than the one we
used to insert it. If we would have re-evaluated $D$ at that time, it would have
been re-inserted into \open{} with a value of $6+3$ and we would have expanded
$E$ first, finding a cheaper path to $D$ before expanding it. In that case, no
reopening happens. We now show that this is the case in general with
re-evaluation.

\begin{theorem}\label{the:dynastar-no-re-exp}
    $\dynastar$ with \reeval{} enabled and using a \dynmonotonic{},
    \dynconsistent{} dynamic heuristic does not reopen states.
\end{theorem}
\begin{proof}
    Assume that a state $s_n$ is reopened at time $\treopen$. Consider the path
    $\tuple{t_1, \dots, t_n}$ to $s_n$ where $t_i = \tuple{s_{i-1}, \ell_i, s_i}$
    is the parent pointer of $s_i$ at time $\treopen$ for all $0 < i \le n$.
    State $s_n$ must have been added to \closed{} at some earlier time $\tclose
    < \treopen$. Since $g^\tclose(s_0) = \Tcost(\tuple{}) = 0$ and
    $g^\tclose(s_n) > \Tcost(\tuple{t_1, \dots, t_n})$, there must be a smallest
    $j$ such that $s_{j+1}$ satisfies $g^\tclose(s_{j+1}) > \Tcost(\tuple{t_1,
    \dots, t_{j+1}})$. For all $s_i$ with $i \le j$, we then know $g^\tclose(s_i)
    = \Tcost(\tuple{t_1, \dots, t_j})$.

    Note that $h^\tclose(s_n) < \infty$ and thus $h^\tclose(s_i) < \infty$ for
    all $i < n$ due to \dynconsistency. Also, since $h$ is \dynmonotonic,
    $h^t(s_i) < \infty$ for all times $t < \tclose$.

    At some time $\texp < \tclose$ the $g$-value of $s_j$ was set to
    $\Tcost(\tuple{t_1, \dots, t_j})$ and at that time, state $s_j$ must have
    been added to \open{}. Between $\texp$ and $\tclose$ the algorithm could not
    have expanded $s_j$ because that would explore transition $t_{j+1}$ and
    contradict $g^\tclose(s_{j+1}) > \Tcost(\tuple{t_1, \dots, t_{j+1}})$. Any
    re-evaluation of $s_j$ that occurs between $\texp$ and $\tclose$ will
    insert $s_j$ back into \open{} because $h^\texp(s_j) < \infty$. Thus at time
    $\tclose$ there is an entry $\tuple{s_j, g^\tre(s_j), h^\tre(s_j)}$ on
    \open{} that was added at a time $\tre$ with $\texp \le \tre < \tclose$.

    At time $\tclose$ an entry $\tuple{s_n, g^\tins(s_n), h^\tins(s_n)}$ is
    popped and expanded that was added earlier at time $\tins < \tclose$.
    \begin{align}
        &g^\tins(s_n) + h^\tclose(s_n) \nonumber\\
        \le{} &g^\tins(s_n) + h^\tins(s_n) \label{eq:dynastar-no-reexp:not-reevaluated}\\
        \le{} &g^\tre(s_j) + h^\tre(s_j) \label{eq:dynastar-no-reexp:min-f}\\
        ={} &\Tcost(\tuple{t_1, \dots, t_j}) + h^\tre(s_j) \label{eq:dynastar-no-reexp:def-j}\\
        \le{} &\Tcost(\tuple{t_1, \dots, t_j}) + h^\tclose(s_j) \label{eq:dynastar-no-reexp:monotonicity}\\
        \le{} &\Tcost(\tuple{t_1, \dots, t_j}) + \Tcost(\tuple{t_{j+1}, \dots, t_n}) + h^\tclose(s_n) \label{eq:dynastar-no-reexp:consistency}\\
        \le{} &g^\treopen(s_n) + h^\tclose(s_n) \label{eq:xy}
    \end{align}
    Step~\eqref{eq:dynastar-no-reexp:not-reevaluated} holds because $s_n$ was
    expanded and not re-evaluated at time $\tclose$;
    step~\eqref{eq:dynastar-no-reexp:min-f} holds because $\tuple{s_j,
    g^\tre(s_j), h^\tre(s_j)}$ was on \open{} but did not have minimal $f$-value
    at time $\tclose$; step~\eqref{eq:dynastar-no-reexp:def-j} holds because the
    $g$-value of $s_j$ matches $\Tcost(\tuple{t_1, \dots, t_j})$ for all times
    between $\texp$ and $\tclose$;
    step~\eqref{eq:dynastar-no-reexp:monotonicity} holds by
    \dynmonotonicity; step~\eqref{eq:dynastar-no-reexp:consistency} holds by
    \dynconsistency; and step~\eqref{eq:xy} holds because the cost of the path
    defined through parent pointers is bounded by the stored $g$-values.

    In summary, we have $g^\tins(s_n) + h^\tclose(s_n) \le g^\treopen(s_n) +
    h^\tclose(s_n)$, so $g^\tins(s_n) \le g^\treopen(s_n)$. This contradicts
    that $s_n$ is reopened at time $\treopen$.
\end{proof}
Note that this proof only relies on \dynconsistency{} between known states so
this result also holds for partially \dynconsistent{} heuristics. Furthermore,
\citet{koyfman-et-al-aamas2024} define the property OPTEX for algorithms where
every expanded state has an optimal $g$-value. If the conditions of
Theorem~\ref{the:dynastar-no-re-exp} are satisfied, $\dynastar$ is OPTEX.

\section{Applications}

We now show how existing approaches, primarily from classical planning, fit our
framework. We argue more formally for the first application and go into less
detail for the others.

\subsection{Interleaved Search}

\citet{franco-torralba-socs2019} consider \emph{interleaved search}, where
search and computation of an abstraction heuristic are alternated in
progressively longer time slices. They refer to a generic function
\texttt{HeuristicImprovement} (\texttt{HI}) that returns an improved abstraction
within a given time limit, potentially using additional information gained
during search.

We represent this idea as a dynamic heuristic by first defining the information
source $\srcis$. Given a function \textup{\texttt{HI}} that takes a time limit
and information from an information source $\srcadd$, the information space
$\infospace[\srcis]$ contains tuples $\tuple{t, l, \alpha, D, \info}$, where $t$
tracks search time, $l$ is the current time limit, $\alpha$ the current
abstraction with abstract distance table $D$, and $\info \in
\infospace[\srcadd]$ is additional information.
The constant $\Sinitinfo_\srcis$ is $\tuple{0, l_0, \alpha_0, D_0,
\Sinitinfo_{\srcadd}}$ where \texttt{HI}$(l_0, \Sinitinfo_{\srcadd})$ returned
$\alpha_0$ and $D_0$.
Calling $\Supdate_{\srcis}$ just replaces $\info$ with
$\Supdate_{\srcadd}(\info)$ and leaves the other components unchanged.
Calling $\Srefine_{\srcis}(\tuple{t, l, \alpha, D, \info_{\srcadd}}, s)$ adds the time
passed since the last call to $\Srefine_{\srcis}$ (or since program start) to
$t$. If $t > l$, it sets $\alpha$ and $D$ to the result of \texttt{HI}$(l,
\info_{\srcadd})$, sets $l$ to $2\cdot l$, and finally sets $t$ to $0$.
The dynamic heuristic $\his$ then calculates the heuristic value
of a state $s$ as $D(\alpha(s))$.

The properties of $\his$ naturally depend on \texttt{HI}.
\citeauthor{franco-torralba-socs2019} describe an implementation based on
symbolic pattern databases \citep[PDBs;][]{edelkamp-aips2002}, let us call it
$\hisft$, that samples states from the open list to guide the pattern selection.
We can approximate this with a nested information source $\srcadd$ that tracks
sampled states in $\Supdate_{\srcadd}$.
They enforce \dynmonotonicity{} by maximizing over known PDBs,
\dynadmissibility{} and \dynconsistency{} are also guaranteed as each
improvement call returns either a partial PDB \citep{anderson-et-al-sara2007} or
a full PDB, making the resulting heuristics admissible and consistent.

\begin{theorem}\label{the:interleaved-optimality}
    Interleaved search as described by \citet{franco-torralba-socs2019} is
    optimal and does not reopen states.
\end{theorem}
\begin{proof}
    $\dynastar$ with \reeval{} using $\hisft$ with the
    \texttt{HI} function defined by \citeauthor{franco-torralba-socs2019} models
    the behavior of interleaved search. Optimality then follows from
    Theorem~\ref{the:dynastar-optimality}
    together with the \dynadmissibility{} of $\hisft$. Finally no states are
    reopened due to Theorem~\ref{the:dynastar-no-re-exp} together with the
    \dynmonotonicity{} and \dynconsistency{} of $\hisft$.
\end{proof}

This confirms \citeauthor{franco-torralba-socs2019}'s conjecture that
re-evaluation and monotonicity (plus \dynadmissibility{} and \dynconsistency{})
ensure that no states are reopened.

\subsection{Online Cartesian Abstraction Refinement}

\citet{eifler-fickert-socs2018} refine additive Cartesian abstractions
\citep{seipp-helmert-icaps2014}, combined via cost partitioning, during search.
Their refinement strategy guarantees \dynadmissibility{}, \dynconsistency{}, and
\dynmonotonicity{}. It is called for states popped from \open{} when a Bellman
optimality equation detects a local error. The resulting heuristic is then used
in $\astar$ with re-evaluation. We can define a suitable information source in a
way analogous to $\srcis$ and show a result similar to
Theorem~\ref{the:interleaved-optimality}.
\begin{theorem}
    $\astar$ with re-evaluation using online Cartesian abstraction refinement as
    described by \citet{eifler-fickert-socs2018} is optimal and does not reopen
    states.
\end{theorem}

\subsection{Landmark Progression}

We have already seen a simplified version of landmark progression as our running
example. In general landmarks are not limited to single operators but denote
properties that hold along all plans of a given task. Since landmarks cannot be
computed efficiently in general \citep{hoffmann-et-al-jair2004}, applications
generate a set of landmarks only for the initial state and then progress this
information \cite[\eg{},][]{richter-westphal-jair2010,
karpas-domshlak-ijcai2009, domshlak-et-al-ipc2011b}.

\citet{buechner-et-al-icaps2023} formalize landmark progression in the landmark
best-first search (LM-BFS) framework, which describes how landmark information
is represented, initialized, progressed, and merged. These components can be
directly cast as a progression source and thus define a progression-based
information source. We can then define a dynamic heuristic based on this
landmark information. LM-BFS can thus be viewed as an instantiation of
$\dynastar$ with re-evaluation, also making our framework applicable to its
instantiations, \eg{} LM-$\astar$ \citep{karpas-domshlak-ijcai2009}.

\begin{theorem}
    An instantiation of LM-BFS using a \dynadmissible{} landmark heuristic and a
    priority queue ordered by $f=g+h$ is optimal.
\end{theorem}

\subsection{\boldmath LTL-$\astar$}

\citet{simon-roeger-socs2015} describe an $\astar$ search where each state $s$
is annotated with a linear temporal logic \cite[LTL;][]{pnueli-focs1977}
formula representing a \emph{trajectory constraint} that must be satisfied by
any optimal path from $s$ to a goal state.

We can define a progression source over LTL formulas and associated $g$-values,
where progression uses the rules by \citet{bacchus-kabanza-aij2000}.
\citeauthor{simon-roeger-socs2015} define the merging of two formulas $\varphi$
and $\psi$ as $\varphi \land \psi$ if they were reached with the same $g$-value,
and as the formula with lower $g$-value otherwise. We can do so by tracking
$g$-values when progressing and merging.

While this gives us a progression-based information source, the heuristic
described by \citeauthor{simon-roeger-socs2015} only guarantees
path-admissibility \citep{karpas-domshlak-icaps2012} for some given information,
from which neither \dynadmissibility{} nor PDA follow directly.

\subsection{\boldmath$\astar$ with Lazy Heuristic Evaluation}

\citet{zhang-bacchus-aaai2012} introduce a modification of $\astar$ that uses
two heuristics, say $h_C$ and $h_A$, such that $h_C$ is cheap and $h_A$ is
accurate. When first inserting a state into \open{}, it is evaluated using
$h_C$. When a state is popped, it is only expanded if its assigned $h$-value was
calculated by $h_A$, otherwise it is evaluated using $h_A$ and re-inserted
into \open{}.

We represent this idea as a dynamic heuristic by first defining the information
source $\srclazy$ over the space of functions $\info\colon \Tstates \to \{C,
A\}$ as $\Sinitinfo_{\srclazy} = \{s \mapsto C \mid s \in \Tstates\}$,
$\Supdate_{\srclazy}(\info, t) = \info$, and $\Srefine_{\srclazy}(\info, s) =
\info'$ such that $\info' = \info$ except $\info'(s) = A$.

The \emph{lazy evaluation heuristic} depending on $\srclazy$ then maps a state
$s$ and information $\info$ to $h_{\info(s)}(s)$.

\begin{theorem}\label{the:lazy-evaluation-optimality}
    $\astar$ using lazy heuristic evaluation with two admissible heuristics is
    optimal.
\end{theorem}
\begin{proof}
    $\dynastar$ with \reeval{} using the lazy evaluation heuristic models the
    behavior of $\astar$ with lazy heuristic evaluation. \dynadmissibility{} is
    given as the two static heuristics are admissible. Applying
    Theorem~\ref{the:dynastar-optimality} concludes the proof.
\end{proof}

The instantiation discussed by \citet{zhang-bacchus-aaai2012} using LM-Cut and
\textsc{Maxsat} is optimal as it satisfies the conditions of the Theorem. We
expect that the same result can be shown for similar approaches such as
selective max \citep{domshlak-et-al-jair2012} or rational lazy $\astar$
\citep{tolpin-et-al-ijcai2013}. Note that using consistent heuristics does not
guarantee \dynconsistency{} of the resulting lazy evaluation heuristic, thus
reopening may be necessary.

\subsection{Path-Dependent $f$-Values}

\citet{dechter-pearl-jacm1985} investigate the optimality of best-first search
algorithms where the $f$-value of a state can depend on the path to that state.
This is similar to our notion of dynamic heuristics although they consider
static heuristics and allow path-dependent evaluation functions $f(\pi)$ other
than $f = g + h$.
Such an evaluation function $f(\pi)$ can be interpreted in our framework as a
dynamic heuristic $\hdyn(s, \pi) = f(\pi) - g(s)$ where paths are stored in an
information source that is updated with newly discovered paths.

Dynamic heuristics can accumulate information based on all paths leading to a
state, even more expensive ones. This is not possible in their framework. In
contrast, their framework covers cases like weighted $\astar$
\cite{pohl-aij1970} that fit our general framework but not $\dynastar$. While we
could encode them into a dynamic heuristic as discussed above, it might not be
useful to investigate the algorithm's behavior in such cases. Generalizing their
results to dynamic heuristics is an interesting line of future work.

%
%

\subsection{Future Work}
There are more applications that likely fit our framework.

Pathmax \citep{mero-aij1984} is a technique to propagate $h$-values from a
transition's origin to its target and vice versa. Originally defined for the
algorithm B', \citet{zhang-et-al-ijcai2009} show how pathmax can be applied to
$\astar$, alongside the more powerful bidirectional pathmax by
\citet{felner-et-al-ijcai2005}.

Deferred evaluation \citep{helmert-jair2006} assigns the heuristic value of
the state being expanded to its successors upon generating them. Once a
successors is expanded itself, its true heuristic value is calculated and
forwarded to its successors.

Partial expansion $\astar$ \citep{yoshizumi-et-al-aaai2000} avoids inserting too
many nodes into the open list by only doing so for promising successors.
Unpromising successors are represented by re-inserting their parent node with
updated priority. This technique does not match the $\dynastar$ instantiation
but it probably fits our general framework.

\section{Conclusion}

We investigated dynamic heuristics and explicitly modeled information they are
based on. We have shown soundness and completeness results for a dynamic
heuristic search
framework tracking such information and looked into $\dynastar$, an
instantiation extending $\astar$ with dynamic heuristics. $\dynastar$ is optimal
with and without re-evaluation when using a dynamic heuristic that is
\dynadmissible{}, a natural extension of admissibility. With re-evaluation, it
further does not reopen states when using a heuristic that monotonically
improves its values and is the dynamic equivalent to consistent. These results
extend the classic results for static heuristics that many existing approaches
appeal to despite using dynamic heuristics. We have seen that such arguments are
not always valid as reopening can occur even with a \dynconsistent{} heuristic.
By showing that existing approaches fit our framework, we provide a formal basis
for such claims.

In addition to the use cases discussed previously,
studying optimal efficiency results of $\dynastar$ would be interesting.

\section{Acknowledgments}
This research was supported by the Swiss National Science Foundation (SNSF) as
part of the project ``Unifying the Theory and Algorithms of Factored State-Space
Search'' (UTA).
\fontsize{9.3}{10.3}\selectfont
\bibliography{abbrv,literatur,crossref}

\end{document}